\documentclass[twoside,leqno,twocolumn]{article}

\usepackage[letterpaper]{geometry}
\usepackage{makecell}
\usepackage[dvipsnames,table,xcdraw]{xcolor}
\usepackage{siamproceedings}

\usepackage[T1]{fontenc}
\usepackage{amsfonts}
\usepackage{url}
\usepackage{graphicx}
\usepackage{epstopdf}
\usepackage{enumitem}
\usepackage{algorithmic}
\usepackage{booktabs}
\ifpdf
  \DeclareGraphicsExtensions{.eps,.pdf,.png,.jpg}
\else
  \DeclareGraphicsExtensions{.eps}
\fi

\newsiamremark{remark}{Remark}
\newsiamremark{hypothesis}{Hypothesis}
\crefname{hypothesis}{Hypothesis}{Hypotheses}
\newsiamthm{claim}{Claim}

\usepackage{amsopn}

\definecolor{kaamilcolor}{RGB}{128, 0, 128}   %
\definecolor{audreycolor}{RGB}{255, 140, 0}      %

\begin{document}

\newcommand\relatedversion{}

\title{\Large MINT: Tensor Decomposition on Stacked Recurrence Matrices for Time Series Data Mining\relatedversion}

\author{Kaamil Kaka$^{*,2}$\thanks{Equal contribution.}\thanks{University of Texas at Dallas, Correspondence: \email{kmk230008@utdallas.edu}}\thanks{$^2$Neuralix AI}
\and Audrey Der$^{*}$\thanks{Work done while at Neuralix AI. Current contact: \email{audrey\_der@intuit.com}}
\and Evangelos E. Papalexakis$^3$
\and Zachary Zimmerman$^3$\thanks{$^3$University of California, Riverside, Department of Computer Science}
\and Vikram Jayaram$^2$}

\date{}

\maketitle

\fancyfoot[R]{\scriptsize{Copyright \textcopyright\ 2026 by SIAM\\
Unauthorized reproduction of this article is prohibited}}

\begin{abstract} Recurrence plots are a time series data mining primitive applied to a variety of domains (e.g. star light curves, sound waveforms, CCT telemetry). This work proposes tensorized self-similarity matrices as a primitive for univariate time series datasets ($N\times n$) of $N$ time series of length $n$ with a subsequence window of length $m$, and whose tensor-based nature is naturally extensible to multivariate datasets. The proposed method to compute this primitive computes dot plots of size $N \times (n-m+1) \times (n-m+ 1)$ from these datasets, where the subsequent tensor is mined using tensor decomposition methods to mine for co-clustered patterns. We demonstrate our results in mass rapid transit, electricity demand, wind turbine, and car traffic data, finding the MINT pipeline effectively co-clusters cross-sensor patterns in highly regular datasets containing motifs at regular intervals.
\end{abstract}

\section{Introduction}\label{sec:introduction}

When the same discordant phenomenon appears across dozens of sensors simultaneously, do we have $N$ anomalies or one? 
Which sensors share not just the time stamp of an event, but its \textit{shape}? 
Which subsequences of a sensor's history participate in a cross-series pattern?
These questions are not \textit{prediction} problems, but \textit{structure discovery} problems, requiring interpretable decomposition, not learned representations.

The Matrix Profile~\cite{MatrixProfileI} is a primitive for motif and discord mining, compressing the full self-similarity structure of a time series into an efficient nearest-neighbor distance vector.
This compression is lossy; the Mplot~\cite{Shahcheraghi}, which retains the full pairwise self-similarity distances between all subsequences, contains richer structure than the Matrix Profile alone.
Prior work demonstrates that Mplots support rich visual analysis of individual time series, enabling practitioners to identify motifs, discords, and phase relationships by eye.
However, Mplots alone are not sufficient as an analytical method.
Visual inspection does not scale with dataset size, and with tens or hundreds of series, there is no accompanying method for extracting cross-series structure or summarizing temporal alignment across the collection.
By retaining the full Mplot for each series and stacking them into a tensor, we consider the entire dataset simultaneously.
The resulting tensor exhibits low-rank structure recoverable by interpretable tensor decomposition methods (e.g. CPD, Tucker, etc.), corresponding to co-clustering information that has no natural representation in per-series analysis.

\begin{figure}[h]
    \centering
    \vspace{-10px}
    \includegraphics[width=0.99\linewidth]{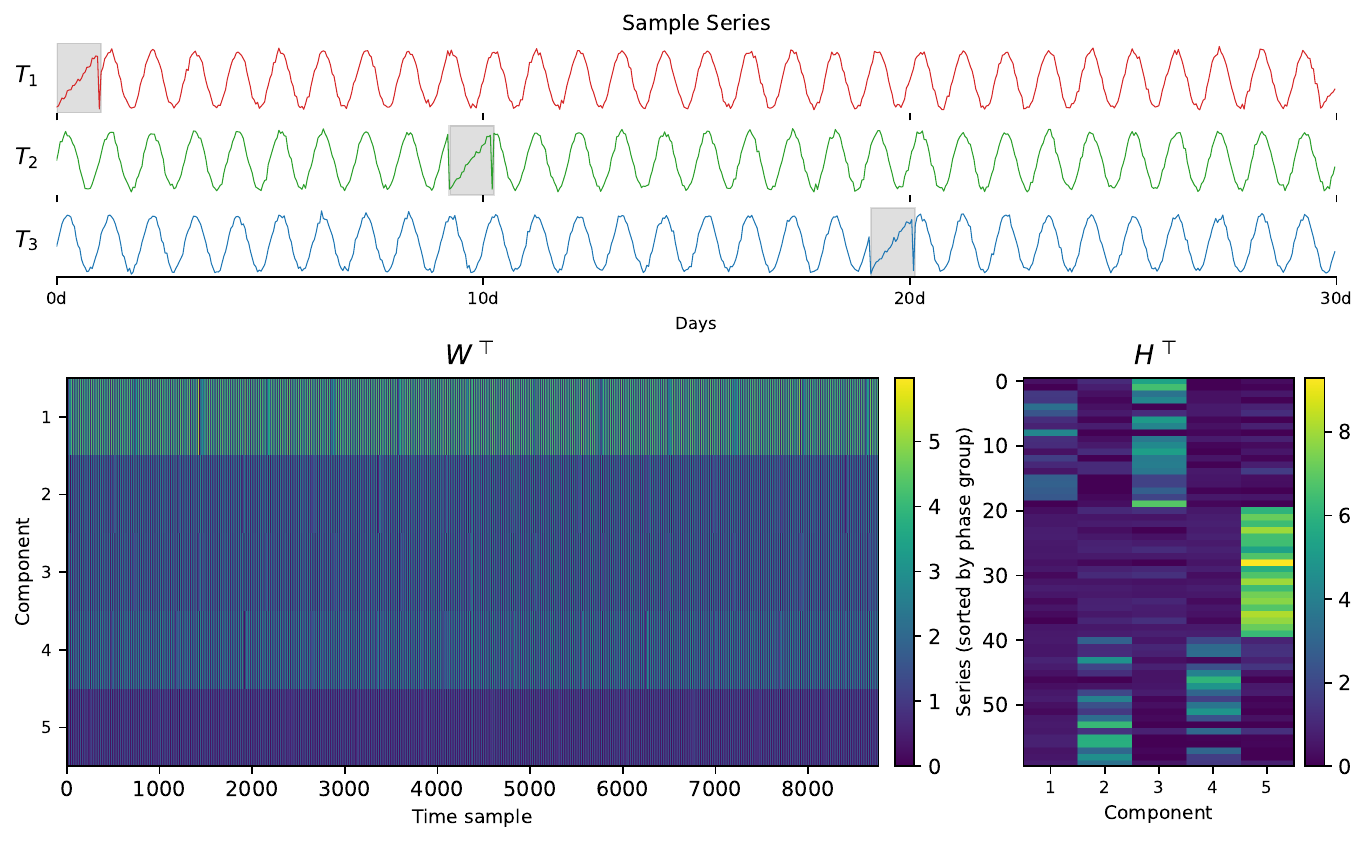}
    \vspace{-10px}
    \caption{NMF applied to the  synthetic dataset.
         \textit{Top)} Representative series from each phase group over the first 30 days, gray shading marks the first sawtooth. \textit{Bottom Left)} The temporal factor $W^{T}$ (time $\times$ components). \textit{Bottom Right)} The series-mode factor $H^T$, columns sorted by ground-truth phase group, revealing a block structure consistent with $T_1$, $T_2$, $T_3$.} 
    \label{fig:NMF}
\end{figure}
To illustrate what is lost when the self-similarity representation is abandoned in favor of scalable analytical methods, consider applying NMF directly to a raw data matrix of 60 synthetic time series, each a sinusoidal signal with additive noise and periodic sawtooth anomalies (Figure~\ref{fig:NMF}).
The series belong to three latent groups, $T_1$, $T_2$, and $T_3$, differing only in the phase offset at which anomalies begin: 0, 10, and 20 days respectively.
This is a controlled setting where the ground truth is known, making it possible to evaluate what a method recovers and what it does not.
NMF recovers the sensor-mode grouping structure in $H^T$: when series are sorted by ground-truth group membership, a clear block pattern emerges.
However, $W$, operating in raw time space, does not surface when each group's characteristic pattern occurs: a 10-day phase difference spans only 240 samples out of 8,761, indistinguishable at figure scale without additional post-hoc analysis.

We now construct and decompose a tensorized self-similarity matrix on the same dataset. Figure~\ref{fig:MINT_works} shows the resulting factor matrices $A$, $B$, and $C$.
The structure is legible and internally consistent across all three modes.
In $A$, the three phase groups emerge as distinct blocks when sorted by ground-truth membership.
This sensor-mode picture is confirmed by the temporal factors: in $B$ and $C$, the bright bands align with the analytically computed phase-offset positions for $T_1$, $T_2$, and $T_3$; these dashed lines are drawn after the fact from ground truth as a validation step, not as inputs to the decomposition.
All three matrices independently arrive at the same component-to-phase assignment, providing cross-modal validation that the decomposition is recovering genuine periodic structure rather than an artifact of the representation.

\begin{figure}[h]
    \centering
    \includegraphics[width=0.99\linewidth]{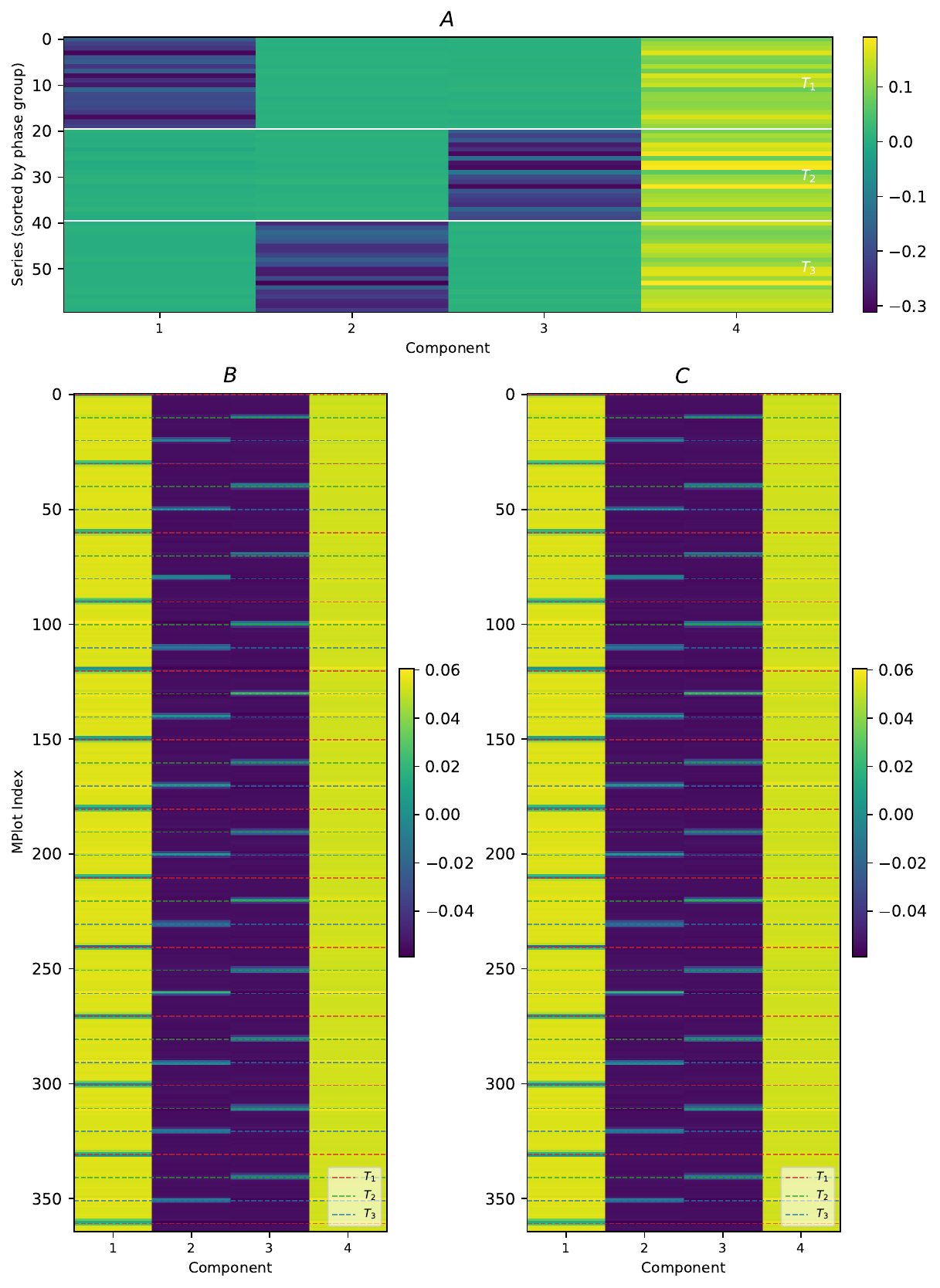}
    \vspace{-20px}
    \caption{MINT applied to the same dataset. \textit{Top)} The sensor-mode factor $A$, rows sorted by phase group, recovering the same three-group block structure as NMF's $H$. \textit{Bottom)} The temporal factors $B$ and $C$, operating in Mplot self-similarity space. Dashed lines mark analytically computed phase-offset positions for $T_1$ (red), $T_2$ (green), and $T_3$ (blue). Bright bands in $B$ and $C$ align with these positions, directly encoding when each group's sawtooth pattern recurs.}
    \label{fig:MINT_works}
\end{figure}

The contribution of this work is not a new decomposition method; rather, it proposes a new object to decompose, the tensor of self-similarity matrices.
We introduce tensorized self-similarity matrices as a new time series data mining primitive, and propose MINT as a first method instantiating this primitive under Euclidean distance. 
We show this representation concentrates cross-series temporal structure in a form that standard tensor decomposition can extract, enabling analysis that was previously limited to what a practitioner could see by eye in individual Mplots.

\section{Motivation and Related Work}\label{sec:related_work}

\paragraph{Recurrence plots and Mplots.}
The essential insight of encoding a time series' self-similarity data in a matrix has a rich history of iterative development and application in many fields. In dynamical systems analysis, \textit{recurrence plots} (RPs) utilize pairwise state similarities to succinctly characterize a nonlinear time series' dynamics \cite{goswami2019brief}. Since Eckmann et al. \cite{Eckmann_1987} introduced the concept in 1987, extensions such as the cross recurrence plot (CRP) \cite{marwan2002} and joint recurrence plot (JRP) \cite{Romano2004MultivariateRP} have generalized its scope to bivariate and multivariate time series, respectively. RP-based methods are tailored to analyzing and comparing recurrence behaviors among time series and have been used effectively in time series classification and anomaly detection \cite{TSCDotPlots,chen2020cnn_rp_anomaly}. Nevertheless, they are sensitive to thresholding and choice of embedding, and thus do not alone provide a single, continuous representation of both intra-series and inter-series self-similarity.

Recently, Kang et al. proposed the RP-GCN methodology \cite{kang2025rpgcn}, which exploits the topological recurrence properties of the RP for time series classification. Given a multivariate time series, an RP is computed for each univariate time series. The RPs are then stacked along a new axis to produce a $3$-tensor, whose graph representation is finally processed by a graph convolutional network (GCN). While RP-GCN captures the recurrence structure of a time series dataset, it only uses it as an intermediate representation to produce a global classification result, not as a mineable object for structure discovery. It does not perform subsequence-granularity co-clustering, nor can its final result be easily mined to provide such information.

Mplots \cite{Shahcheraghi} leverage the Matrix Profile as a recurrence plot, finding patterns in star light curves, animal sound waveforms, CCT telemetry in mice, and pulsus paradoxus (an exaggerated drop in blood pressure). A more formal definition is given in \cref{sec:methodology}, but an Mplot is roughly a matrix whose entries are the Euclidean distances between every pair of subsequences. Intuitively, Mplots measure the self-similarity of a time series in Euclidean Distance sliding window fashion; they cluster correlated subsequences in a single series, thus capturing key structural information. Moreover, they are interpretable, enabling visual observation of time series trends, motifs, and discords. For our objectives, the Mplot is thus the richest object in this line of work, but it has never been treated as a mineable \textit{dataset-level} (rather than \textit{series-level}) primitive, either for univariate or multivariate datasets.

\paragraph{Tensor methods for time series.} Tensor algebra algorithms historically developed in tandem with applications in psychometrics, chemometrics, and signal processing, but in the past two decades they have borne substantial fruit in machine learning and data mining \cite{sidiropoulos2017tensor}. Prior work on tensor methods for time series largely treats tensors as the \textit{observed} data structure (e.g., an $n$-tensor time series whose entries may be stacked to form a $(n+1)$-tensor), modeling their dynamics for forecasting or dimensionality reduction through tensor decomposition (e.g. CP or Tucker) or tensor autoregressive models \cite{bolivar2026tensor, changMatrixTS}. Cross-series self-similarity structure is absent from these formulations entirely, hindering their ability to capture subtle co-clustering information at the subsequence level.
\paragraph{Time series clustering.} One of the fundamental tasks in time series analysis is \textit{clustering}, partitioning data samples into groups in such a way that in-group similarity is maximized and inter-group similarity is minimized. Individual clusters share domain-relevant features and thus constitute potential \textit{patterns} revealing the latent structure of the time series~\cite{paparrizos2024bridging}.

Whole time-series clustering (WSC) learns clusters of individual series within a collection of distinct time series. Paparrizos et al.~\cite{paparrizos2024bridging} propose a comprehensive taxonomy of this space incorporating both classical and deep learning strategies, classifying clustering techniques into (i) distance-based, (ii) distribution-based, (iii) subsequence-based, and (iv) representation-learning-based methods. Subsequence-based WSC methods in particular extract representative sequences to differentiate among time series clusters. While they may leverage similarity data at the subsequence level—for example, the Matrix Profile in MPNCMI clustering~\cite{MPNCMI}—they ultimately associate only a single label to each time series. Moreover, they do not natively detect time-aligned correlations between subsequences in different series.

Subsequence-based WSC is distinguished from subsequence clustering, which clusters time series subsequences obtained by sliding windows rather than the individual time series themselves~\cite{paparrizos2024bridging}. It is important to distinguish subsequence clustering from subsequence \textit{co}-clustering. Keogh and Lin~\cite{subsequenceClustering} demonstrated that, for virtually all datasets, the results obtained by subsequence clustering are essentially random. This ``meaninglessness'' is an inextricable result of the way subsequences are sampled by overlapping sliding windows, rather than features of the input data itself. Subsequence \textit{co}-clustering, on the other hand, clusters patterns along several modes (in both time and sensor space, in our case), rather than just in the time domain. We leave formal verification for future work, but we propose that the tensorized self-similarity matrix (and the MINT pipeline) constitutes an existence proof that this setting forces clustering methods to find genuine co-occurrence.

\paragraph{Conclusion.} None of these works, individually or combined, produces an interpretable, mineable representation of cross-series subsequence self-similarity structure.

\section{Definitions and Methodology}\label{sec:methodology}
A \textbf{\textit{time series}} $T=\{t_1, t_2, ..., t_n\}$ is a sequence of real numbers, and a \textbf{\textit{subsequence}} $T[i,j]$ of length $m$ is a contiguous subset of values from $T$ starting at index $i$ and ending at $j$, where $j=i+m - 1$.
A \textbf{\textit{distance profile}} $DP_{AB}^{j,m}$ is the vector of distances between each subsequence in a reference time series $T_A$ and a query subsequence $T_B[j,j  + m - 1]$.
We show examples of these primitives in Figure~\ref{fig:tqdp}.

\begin{figure}
    \centering
    \vspace{-15pt}
    \includegraphics[width=0.99\linewidth]{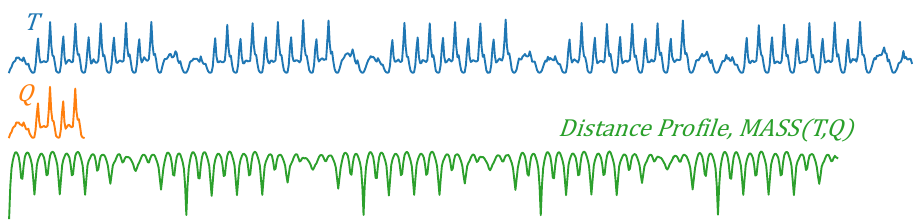}
    \caption{\textit{Top)} A time series \textcolor{MidnightBlue}{T} of Mass Rapid Transit usage. \textit{Middle)} A subsequence \textcolor{orange}{Q} from \textcolor{MidnightBlue}{T} \textit{Bottom)} A \textcolor{OliveGreen}{distance profile} computed using Mueen's Algorithm for Similarity Search~\cite{MASS} of \textcolor{orange}{Q} against \textcolor{MidnightBlue}{T}.}
    \label{fig:tqdp}
    \vspace{-10pt}
\end{figure}

In this work, we consider Self-Join Mplots, which leverage a special case of an AB-Join Matrix Profile, where $T_A$ and $T_B$ are the same time series~\cite{Shahcheraghi}.
\begin{definition}
    An \textbf{Mplot} is the visualized matrix of distance profiles. Each row of this matrix is $DP^{j,m}_{AB}$.
\end{definition}

There are resources available on how to read dot/recurrence plots~\cite{MARWAN2007237_readingdotplots,rakthanmanon2013addressing,march2005recurrence,Shahcheraghi}, but to keep this work self-contained, we  reiterate a few key points, discussing ``light values'' as having low similarity and ``dark values'' having high similarity:
\begin{itemize}
    \item The Mplot is a subsequence self-similarity matrix; consequently, a visible central diagonal has a distance of zero.
    \item The appearance of evenly spaced diagonals parallel to the central diagonal indicates periodicity or cycles in the time series.
    \item Mplots can result in a type of pattern where the slope changes, either ``up'' or ``down'', indicating patterns are either slowing down or speeding up. 
    \item A streak that curves in both directions suggests a pair of subsequences that match after one locally “warps” in order to match the other. This is a useful characteristic that allows us to spot patterns in a warping-invariant manner. 
\end{itemize}
The key novelty in this work is to treat Mplots as decomposable matrices and to extend this perspective to decomposing constructed stacked Mplots using classical tensor decomposition methods.
\begin{definition}
    An $a \times b$  matrix is a function of two integer indices; that is, a function $M: [1..a] \times [1..b] \to \mathbb{R}$. We view a matrix as a ``2-dimensional array'' where $M(i,j)$ denotes the entry at the $i$th row and $j$th column.
\end{definition}

Given an $a \times b$ matrix $M$, 
\begin{enumerate}

    \item Its \textit{entrywise $\ell_1$ norm} is the sum of the absolute values of its entries:

    \[||M||_{1} = \sum_{i =1}^a\sum_{j = 1}^b |M(i,j)|.\]
    
    \item Its \textit{nuclear norm} is the sum of its singular values:

    \[||M||_{*} = \sum_{i =1}^{\min(a,b)} \sigma_i(M).\]
    
\end{enumerate}

\begin{definition}
    A \textbf{tensor} of order $n$ is a function of $n$ integer indices; that is, a function $T: [1 .. k_1] \times...\times[1..k_n] \to \mathbb{R} $. The positive integer $k_l$ is the size of the tensor along the $l$th mode or axis.
    
\end{definition}

Informally, a tensor of order $n$ is an ``$n$-dimensional matrix.''
We say a tensor $T$ is \textit{rank-$1$} if it is the outer product of $n$ vectors.

\[T= \mathbf{a}_1 \otimes ... \otimes \mathbf{a}_n.\]

On a similar line, we say that $T$ is \textit{rank-$R$} if it can be expressed using a minimum of $R$ outer products of $n$ vectors \cite{sidiropoulos2017tensor}.

\[T = \sum_{i=1}^R \mathbf{a}_{1,i} \otimes ... \otimes \mathbf{a}_{n,i}.\]

Our analysis primarily concerns tensors constructed by stacking Mplots along one axis. Our use of the term \textit{tensor} refers to a third-order tensor (3-tensor) unless otherwise specified.
For a given time series dataset $\mathcal{D} = \{T_1, T_2, ..., T_N\}$, the resulting tensors constructed from stacked Mplots are of size $N \times (n -m+1)\times (n-m+1)$, where $n$ is the length of the time series in the dataset, $m$ the subsequence length parameterizing the Mplots, and $N$ the number of time series in $\mathcal{D}$. The methods used to analyze the objects defined above are:

\textbf{Nonnegative Robust PCA}. We adopt a Robust PCA formulation derived from the Stable PCP of Zhou et al. \cite{stablePCP}. Given an $n \times n$ matrix $M$, we find a nonnegative low rank matrix $L$ and a sparse matrix $S$ such that
    \[M = L + S + Z, \]
    where $Z$ is a bounded entry-wise noise term. If $\lambda = \frac{1}{\sqrt{n}}$ is the standard RPCA sparsity parameter, then we consider the problem
    \[\min_{L \geq 0, S} ||L||_{*} +\lambda||S||_1 + \frac{\mu}{2} ||M-L-S||^2_F.\]
    We obtain a heuristic estimate $(\hat{L}, \hat{S})$ by alternating minimization. Initially, $\hat{L} = \hat{S} = 0$. In each iteration, $\hat{L}$ is computed by singular-value thresholding of $M-\hat{S}$ at level $\frac{1}{\mu}$ followed by a nonnegativity clamp, while $\hat{S}$ is computed by soft-thresholding of $M - \hat{L}$ at level $\frac{\lambda}{\mu}$. We iterate until the relative iterate change in $\hat{L}$ falls below $\epsilon = 10^{-5}$ or after a maximum of $5000$ iterations.

    Assume $Z$ has i.i.d. $N(0, \sigma^2)$ entries. With penalty $\frac{1}{2\mu}||M - L - S||_F^2$,  Zhou et al.  prescribe $\mu = \sqrt{2n}\sigma$. Our penalty is $\frac{\mu}{2} ||M-L-S||^2_F$, so the equivalent choice is $\mu = 1/(\sqrt{2n}\sigma)$. We inherit the noise-calibrated choice of $\mu$ from the original formulation \cite{stablePCP}. MINT adds a nonnegativity constraint on $L$, and we adopt $\mu$ (which tunes the residual term to the noise level $\sigma$) as a reasonable heuristic for our form of the objective.\footnote{Zhou et al. establish stable recovery guarantees for Stable PCP
    under the same conditions as PCP. We leave analogous guarantees for the nonnegative-constrained variant and our projected iteration to future work.}
    
    It remains to estimate $\sigma$ given $M$. We accomplish this through a two-step algorithm.
    \begin{enumerate}
    \item We recover a coarse low-rank approximation $M_\tau$  by hard-thresholding singular values at the optimal threshold $\tau = 2.858 \cdot y_{med}$, where $y_{med}$ is the median empirical singular value \cite{optimalHardThreshold}.  We then estimate the entry-wise noise by the residual $Z_\tau =  M - M_\tau$.
    \item Because $Z_\tau$ is  effectively symmetric, each off-diagonal residual appears twice; we compute the median absolute deviation (MAD) over its strict upper triangular values to count each value once. Then we estimate 
    \[\hat{\sigma} = \frac{\text{MAD}}{0.6745}.\] This parallels the formulation in Sato and Ono, Section IV-I \cite{sato2026joint}.
    \end{enumerate}
    
\textbf{Canonical polyadic decomposition (CPD)}. Given an order-$n$ tensor $T$ and a rank $R$, CPD returns $n$ matrices
    \[A_1 = [\mathbf{a}_{1,1}, ...,\mathbf{a}_{1,R}],\]
    \[...\]
    \[A_n = [\mathbf{a}_{n,1}, ...,\mathbf{a}_{n,R}],\]
    such that
    \[T = \sum_{i=1}^R \mathbf{a}_{1,i} \otimes ... \otimes \mathbf{a}_{n,i}.\]
    In other words, the $k$th matrix gives the $k$th terms of each outer product in the outer product decomposition \cite{sidiropoulos2017tensor}.  
    We interpret each of the $R$ columns in the matrices as representing a unique \textit{component} or \textit{pattern}. If we stack Mplots along the first axis of a $3$-tensor, then we interpret $A_k(i,j)$ as the expression of component $j$ in Mplot $i$ (if $k$ equals $1$) or in the distance profile of subsequence $i$ (if $k$ equals $2$ or $3$). %
We introduce the tensorized self-similarity matrix as a new time series data mining primitive, instantiated by MINT (\textbf{M}plots \textbf{IN}to \textbf{T}ensor) as an initial method of computation.
MINT is described by the pipeline shown in Figure~\ref{fig:bespoke_taipeiMRT_series}:
\begin{enumerate}
    \item \textbf{Data loading and preprocessing.} 
    Given a dataset $\mathcal{D}$, we extract time series $T_1, T_2, ... ,T_N \in \mathcal{D}$, aligned on a common axis with nearly uniform intervals. While implementations of the Matrix Profile can handle missing values in the input data, MINT requires NaNs to be preprocessed.
    \item \textbf{Mplot Computation.} We use the SCAMP library~\cite{ZimmermanSCAMP} to compute a set of Mplots for all $T_i \in \mathcal{D}$. It appears at first glance that the time complexity of Mplot construction is $O(mn^2)$, where $m$ is the subsequence length parameterizing the Mplot computation. However, advances in GPU implementations~\cite{ZimmermanSCAMP,Shahcheraghi} render this process more reflective of $O(n^2)$ in practice. We limit Mplots to a $k \times k$ window, where $k < n$ is a predefined \textit{pooling parameter}; in the SCAMP library, pooling is accomplished using \texttt{max()} for a similarity matrix and \texttt{min()} for a distance matrix. Prior work in hyperspectral imaging \cite{jayaram2004} shows that lossy compression of sensor data can retain classification-relevant structure; we observe the analogous property for min-pooled distance matrices.
    \item \textbf{Tensor Construction.} For each Mplot $M_i$, we use Nonnegative Robust PCA to decompose it into a low-rank matrix $L_i$ and a sparse matrix $S_i$. Using exact SVD, this can be accomplished in approximately $O(n_\text{iter}k^3)$ time, where $n_\text{iter} \leq 5000$ is the number of iterations. In practice,  Robust PCA is the primary bottleneck of the proposed pipeline, particularly with large Mplots.
    The low-rank matrices are mean-centered, then stacked to form a 3-tensor $T$ of dimensions $N \times k \times k$, which is then magnitude-normalized. 
    \item \textbf{CPD Decomposition.} We use TensorLy's~\cite{tensorly} CANDECOMP/PARAFAC method to perform CPD on $T$ in $O(n_\text{iter}k^2N)$ time (where $n_\text{iter}$ is again the number of iterations), thereby producing three matrices 
    \[
    A = [\mathbf{a}_1 ,...,\mathbf{a}_R],
    B = [\mathbf{b}_1 ,...,\mathbf{b}_R],
    C = [\mathbf{c}_1 ,...,\mathbf{c}_R],
    \]
    of dimensions $N \times R$, $k \times R$, and $k \times R$ respectively. The rank $R \in [2..8]$ is estimated by locating the elbow point of the reconstruction error curve using the Kneedle algorithm. If no elbow is found, a flag is raised and the procedure defaults to $2$ components. The process returns three \textit{factor matrices}, which are then further analyzed.
\end{enumerate}

\begin{figure}
    \centering
    \includegraphics[width=1.0\linewidth]{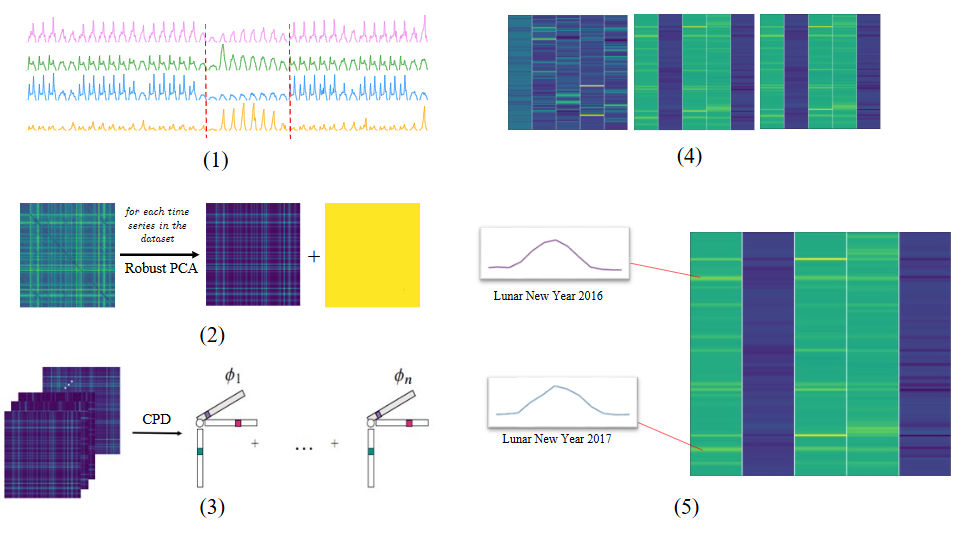}
    \caption{Overview of the MINT pipeline. Analysis of peaks in individual loadings enables easy visual identification of cross-sensor phenomena, such as Lunar New Year festivities.}
    \label{fig:bespoke_taipeiMRT_series}
    \vspace{-10pt}
\end{figure}

This setting is distinct from multivariate time series analysis, where the $d$ dimensions are semantically coupled by construction.
In our setting, univariate series are independent sensors whose co-clustering structure is \textit{latent}, recoverable without labels or assumptions, rather than given.
The reader will also appreciate that this primitive is grounded in tensor mining, and naturally extends to multivariate time series datasets.

It remains to justify mean-centering low-rank matrices in step (3). Each series' Mplot (hence low-rank matrix) has a varying amount of self-similarity, from highly regular or periodic to noisy. A series' baseline self-similarity level, considered as a mean-determined offset of the similarity matrix, might describe it in isolation, but it will not contribute information about when patterns recur or which series also share them. If not subtracted, the offset carries a large Frobenius norm (``energy'') contribution, and CPD will spend a component ranking sensors by overall self-similarity. With a limited rank range, mean-centering of low-rank matrices enables us to surface genuine co-clustering information rather than per-slice global similarity backgrounds.

\section{Datasets}\label{sec:datasets}
We present our results through a set of five case studies on diverse datasets from various domains with different sampling intervals.
The \textbf{Mass Rapid Transit (MRT)} dataset~\cite{GoldenBatchYeh} contains entry and exit ridership data from 108 Taipei stations sampled hourly during operating hours over a total length of 517 days between November 1, 2015 and March 31, 2017.
Sample subsequences from this dataset are shown in Figure~\ref{fig:bespoke_taipeiMRT_series} and Figure~\ref{fig:tqdp}.
Sample time series from the \textbf{Spatio-Temporal Traffic (LargeST)} dataset~\cite{liu2023largest} are shown in Figure~\ref{fig:largeST_transformations}. This dataset consists of traffic data from a total of 8,600 sensors in California, USA over five years, at a base sampling rate of 5 minutes. Prior to analysis, we extracted data from the year 2019 and aggregated them into 15-minute means. The \textbf{Electricity Load Diagrams dataset} \cite{electricity} contains the quarter-hourly electricity consumption (in kW) of 370 clients from 2011 to 2014. We used only 2013, a year in which most clients had complete records, to minimize long initial zero sequences, and converted the data to units of kWh before analysis.

Wind Farm A of the \textbf{CARE to Compare} dataset \cite{caretocompare2024} is based on data from the EDP open platform\footnote{https://www.edp.com/en/innovation/open-data/data} and consists of five wind turbines of an onshore wind farm in Portugal.
The wind farm data contains $22$ datasets with a total of $86$ features. For a given turbine, various sensor data traces are available, including but not limited to the ambient temperature, wind direction, temperature recordings from various turbine components (nacelle controller, gearbox, hub controller, generator, etc.), voltage, and current. 54 unique sensor traces are represented by derived statistical features such as the $10$-minute average, max, min, and standard deviation.

Testing was performed on Dataset 0, comprising one year of training data and a prediction window containing a labeled anomaly window preceding a generator bearing failure, surrounded by padding. Before processing, only the $10$-minute average of each sensor trace is taken, which results in $54$ considered sensors. Sample subsequences from Dataset 0 are shown in Figure~\ref{fig:wind_turbines}.

The \textbf{Electrical Power Demand (OPSD)} dataset\footnote{https://open-power-system-data.org/} is the recorded electrical demand of 17 European countries at hourly sampling intervals. 
Each country has at least one time series, all of which span dates from December 31, 2014 to September 30, 2020, or 50,401 time steps.

\subsection{Case Study: OPSD}
We provide an existence proof for MINT's ability to surface real structure which no single-series method finds. We run the pipeline on the OPSD dataset. Due to a small number of sensors (18), we impute NaN-filled sequences with random noise. From the method, we obtain three factor matrices $A$, $B$, and $C$, which are then visually analyzed for patterns. 

\begin{figure}
    \centering
\includegraphics[width=\linewidth]{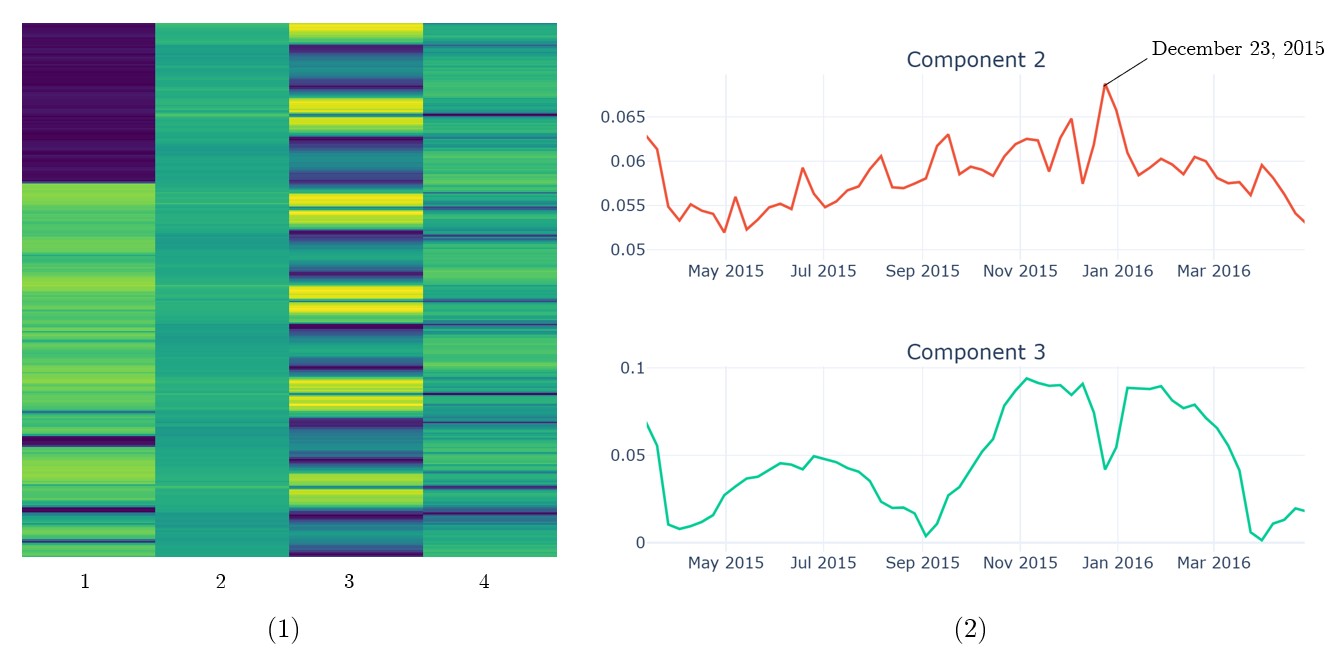}
    \caption{(1) Transformed factor matrix $C$ for the OPSD dataset. (2) Analysis of Components 2 and 3 indicates that they correspond to distinctive Christmas and seasonal electricity demand patterns, respectively.}
    \label{fig:opsd_c_analysis}
    \vspace{-10pt}
\end{figure}

Figure \ref{fig:opsd_c_analysis} depicts this analysis applied on factor matrix $C$ (in particular, its entrywise absolute value $|C|$). The vertical axis corresponds to time, with each pixel signifying a week-length interval. For each component, we identify highly expressed positions (indices $i$ for which $|C(i,j)|$ is a high value) and map them back to dates in the original data. 

From this process, we conclude that the MINT components compartmentalize distinct cross-sensor temporal patterns. For example, two out of the five highest expressed positions in Component 2 correspond to Christmas 2015 and 2019, while two more are in the 2014 and 2016 holiday seasons. Component 3, on the other hand, corresponds to seasonal electricity patterns. Its signal follows an annual schedule with high peaks in winter, a slightly smaller peak in summer, and troughs in the intervening fall and spring. In particular, even though Component 3 troughs during the holiday season, it still attains a background ``winter season'' level that exceeds fall and spring. We expect short-scale holiday phenomena and long-scale seasonal periodicities to correspond to distinct electricity demand signatures, and this is reflected by MINT's component segregation.

The tensorial perspective is crucial: while a single-series method (such as the Matrix Profile) applied on each sensor's time series would detect distinct electricity demand patterns each winter, it would be difficult to distinguish intensities and isolate contributing factors (e.g. cold temperatures versus holiday festivities) without additional contextual information. In contrast, the highlighted components in the MINT tensor  generalize information from diverse time series. Thus, MINT not only identified but also \textit{separated} Christmas events and seasonal patterns across \textit{all} time series in the dataset, regardless of other confounding phenomena that might occur in each time series.

\section{Experimental Setup}\label{sec:experiment_setup}
This work's primary novelty lies in identifying co-clustering phenomena in time series, capturing both seasonality and shapelets simultaneously in $O(N \cdot n^2 +N \cdot k^3)$ time, where  $n$ is the time series length, $k$ is the side length of the pooled Mplots, and $N$ is the number of series transformed into Mplots to construct the tensor. Deep clustering approaches for time series typically rely on learned representations and iterative optimization~\cite{paparrizos2024bridging}, resulting in unpredictable runtime and resource requirements that stand in contrast to the analyzable computational complexity of MINT.
We test our claims to co-clustering through robustness testing by Noise Injection.

\subsection{Single-Trial Procedure}\label{sec:single_trial_procedure}
Let us describe a univariate time series $T$ with a set of $N$ sensors, where the notation $T_i$ indicates the $i$th sensor.
We choose $b<N /2$ and a set $\mathcal{W}$ of disjoint windows of length $m$ from $T$. We construct $\mathcal{W}$ by dividing $T_i$ into windows of length $m$ and randomly selecting $|\mathcal{W}|$ of these windows.
For each trace $T_i \in T$, we randomly assign one of the following transformations $\theta_x$ to produce $\hat{T_i}$:
\begin{itemize}
    \item $b$ time series $T_i \in T$ are \textit{Completely Random} ($\theta_1(T_i)$): $\hat{T_i}$ is reassigned to be random noise. We measure the trace mean and standard deviation to construct the signal of random normal noise.
    \item $b$ time series are \textit{Mostly Random} ($\theta_2(T_i)$): $\hat{T_i}$ is transformed to be identical to its normal time series only in the $|\mathcal{W}|$ disjoint windows, and is otherwise random noise, as described above.
    \item $N - 2b$ time series are \textit{Mostly Normal} ($\theta_3(T_i)$): $\hat{T_i}$ is identical to its normal time series everywhere except the $|\mathcal{W}|$ disjoint windows, which are replaced with random noise.
\end{itemize}
In Figure~\ref{fig:largeST_transformations}, we show the original time series $T_i$ and its three transformed series:  $\theta_1(T_i)$, $\theta_2(T_i)$, $\theta_3(T_i)$. 
After computing the Mplots for all $\hat{T} = \{\hat{T_1}, ..., \hat{T_N}\}$, we decompose the tensor using CPD.
For each resulting component $\phi$ in the decomposition, we obtain a set of top-$s$ sensors $S(\phi)$ from a distinctiveness pipeline.

\begin{figure}
    \centering
\includegraphics[width=\linewidth]{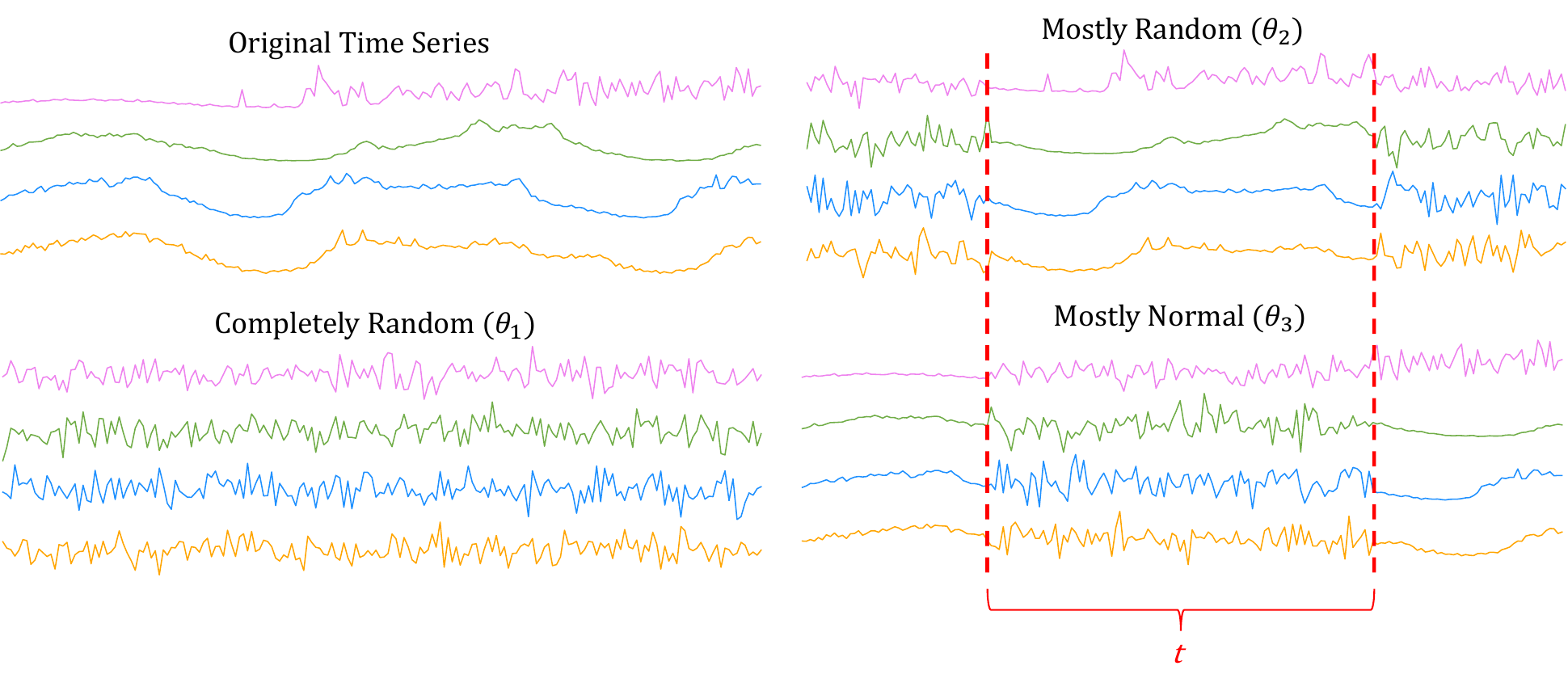}
    \caption{Four subsequences from four time series from the LargeST dataset, surrounding one disjoint window $t$. The Completely Random $\theta_1(T_i)$, Mostly Random $\theta_2(T_i)$, and Mostly Normal $\theta_3(T_i)$ transforms on the original $T_i$ are shown. In Mostly Random subsequences, we see $t$ contains the original data, and the remainder of the series is random noise. In the Mostly Normal series, $t$ is replaced with random noise, and the remainder of the series is the original data.}
    \label{fig:largeST_transformations}
\end{figure}

\subsubsection{Distinctiveness Pipeline}
Given the $i$th component $\phi$, we expect the sensors (or multidimensional traces) $T_j$ for which $|A(j,i)|$ is highest would exhibit the highest ``expression'' of $\phi$. Similarly, among subsequences within the sensors, those indexed by $j$ with the largest $|B(j,i)|$ and $|C(j,i)|$ values are expected to most strongly express $\phi$. Combining these observations, we conclude that the top-$s$ length $m$ subsequences of the top-$s$ sensors represent ``archetypal patterns'' that best capture the structural features encoded by $\phi$. We identify these sequences for all $\phi$ through this \textit{Distinctiveness Pipeline}:
\begin{enumerate}
    \item Use CPD to find the factor matrices \textbf{$A$}, \textbf{$B$}, and \textbf{$C$}.
    \item In component $\phi$ of \textbf{$A$}, find the top-$s$ indices and find the corresponding ``sensor'' names and store in a set $S(\phi)$.
    \item In component $\phi$ of \textbf{$B$}  (or \textbf{$C$}), find the top-$s$ indices and map them to their corresponding time series intervals. To account for pooling effects, we pad these intervals minimally: if an interval $[a,b]$ is mapped and each pixel on the Mplot corresponds to $c_k$ positions, we extend the interval by $c_k$ on both sides, yielding $[a-c_k, b+c_k]$. Store these adjusted intervals in a set $I(\phi)$.
    \item For each interval $[i,j] \in I(\phi)$ and each sensor $\varsigma \in S(\phi)$, consider the subsequence $T_\varsigma[i,j]$.
\end{enumerate}

The choice of $B$ or $C$ is arbitrary because (up to numerical error) they are effectively equivalent; for our experiments, we fix the factor matrix used as $C$.

\subsection{Validation Criteria}
Using $S(\phi)$ and $I(\phi)$, we consider four properties to demonstrate the utility of MINT to find co-clustering information.
\begin{enumerate}
    \item $\mathbf{P_1}$: There exists a $\phi$ such that for all $I \in I(\phi)$ there exists $I' \in \mathcal{W}$ for which $I' \cap I \neq \emptyset$. %
    \item $\mathbf{P_2}$: Each component strongly reflects sensors transformed by exactly one of the transformations $\theta_1$, $\theta_2$, $\theta_3$. That is, for each $\phi$, there exists $x \in \{1,2,3\}$ such that $\theta_x$ \textit{corresponds} to $\phi$, i.e. more than half of the elements of $S(\phi)$ are transformed by the same $\theta_x$.
    \item $\mathbf{P_3}$: There exists at most one component $\phi$ s.t. $S(\phi)$ contains $\theta_1$ (\textit{Completely Random}) sensors.
    \item $\mathbf{P_4}$: $\theta_3$ (\textit{Mostly Normal}) \textit{corresponds} to a \textit{significant majority} of the components. We fix the threshold for this criterion at $\lfloor 0.7 R \rfloor$, where $R$ is the number of components. This choice permits small decompositions (e.g., $R=2$) to remain valid even if only one component is classified as \textit{Mostly Normal}.
\end{enumerate}

Each of the properties reflects a desirable aspect of the cross-mode co-clustering behavior that our pipeline is designed to exhibit. $\mathbf{P}_1$ measures the ability of MINT to account for strong regularities in similar positions across  series.
$\mathbf{P}_2$ ensures similarly transformed  sensors are grouped in similar categories, and $\mathbf{P}_3$ and $\mathbf{P}_4$ formalize the assumption that the dominant sources of variation in the transformed dataset reflect the original data (as opposed to noise imputation).

\subsection{Hypothesis Testing}
In the random-selection model, the top-$s$ sensors and the left endpoints of top-$s$ intervals are selected uniformly and without replacement from $[1 .. N]$ and $[1.. n]$ respectively, independently for each of the $R$ components. Each property is a deterministic function of the resulting category count, so its chance pass rate can be bounded above by a hypergeometric or binomial tail probability. With the experimental constants defined in Section \ref{sec:experiment_details}, chance pass rates are bounded in Appendix \ref{app:chance_pass_rates} and characterized below.
\begin{align*}
p_1^\text{chance} &\le 0.00584,
&
p_2^\text{chance} &\le 0.176 , 
\\
p_3^\text{chance} &\le 0.325,
&
p_4^\text{chance} &\le 0.570.
\end{align*}
Therefore, to differentiate genuine co-clustering robustness from that expected under random selection, we must examine the \textit{proportion} of incidence for each property. To this end, the procedure defined in Section \ref{sec:single_trial_procedure} is run on the same time series $T$ for 50 trials. Across all trials, we compute the proportion of trials for which each of $\mathbf{P}_1, \mathbf{P}_2, \mathbf{P}_3, \text{and } \mathbf{P}_4$ is satisfied (thus obtaining $\hat{p}_1,\hat{p}_2,\hat{p}_3,\text{ and }\hat{p}_4$ respectively). To ensure that a chance process lies inside the null region, we define the hypotheses
\begin{align*}
(H_{0,i}):&~ p_i \leq 0.7, & (H_{1,i}):&~ p_i  > 0.7, & \text{for $i = 1,2,3,4$}.
\end{align*}
We reject or fail to reject each null hypothesis using a one-tailed z-test for proportions with a Bonferroni-corrected significance of $\alpha = 0.0125$.

\subsection{Experiment Details}
\label{sec:experiment_details}
We performed the experiment pipeline on four datasets: (1) Taipei MRT entrance data, (2) 2019 LargeST Traffic data, (3) 2013 Electricity Load Diagrams data, and (4) CARE to Compare wind turbine data. In the four experiments, all sensor columns were used with the exception of those resulting in Mplots containing NaNs. %
This resulted in yields of 108/108 sensors for Taipei MRT, 8303/8600 sensors for LargeST, 332/370 sensors for Electricity Load Diagrams, and 43/54 sensors for CARE to Compare.
In experiments (1) and (4), each trial was tested on all NaN-free sensors, while in experiments (2) and (3), each trial was tested on a sample of $100$ randomly selected NaN-free sensors.

Experiments were conducted on a Windows 10 virtual machine with $16$ vCPUs and $64$ GiB of RAM. The trials were conducted in parallel and independently on separate deterministic seeds. Trials that yielded NaNs in Mplots, when identified, were discarded altogether and re-tried on a new seed. In total, $50$ valid trials were obtained for each dataset.

For the four datasets, we set the subsequence lengths $m_1 = 146$, $m_2 = m_3 = 95$, and $m_4 = 140$, respectively.
These selections correspond to practical use cases in identifying week-length, day-length, and day-length motifs respectively.
Following~\cite{Shahcheraghi}, we select subsequence lengths $m$ to be slightly less than the length of the motifs of interest to help ensure these motifs are visible in the resulting Mplot. For most datasets and practical applications, Mplots are generally forgiving with respect to the exact value of $m$.
We set the pooling parameter $k =  \lceil \frac{3n}{m} \rceil$, so that each Mplot pixel corresponds to $c_k \approx\frac{m}{3}$ positions. We define the co-clustering window set $\mathcal{W}$ such that it covers 10\% of each time series. 
We choose $s = 6$ as the minimum number of top-$s$ sensors for which our four properties remain nontrivial. Finally, we set $b= \lceil 0.25N \rceil$ as the number of sensors transformed by each of $\theta_1$ and $\theta_2$.

\section{Results and Discussion}\label{sec:results}
Table \ref{tab:proportion_estimates} lists the proportions $\hat{p}_i$ for $i = 1,2,3,4$ obtained after conducting $50$ trials on the four datasets tested.

\setlength{\textfloatsep}{5pt} %
\setlength{\floatsep}{5pt}

\begin{table}[ht]
\centering
\footnotesize
\setlength{\tabcolsep}{3pt}
\renewcommand{\arraystretch}{1.6}
\providecommand{\cell}[2]{\makecell{#1\\{\scriptsize ($#2$)}}}
\resizebox{\columnwidth}{!}{%
\begin{tabular}{l c c c c}
\toprule
& \makecell{\textbf{Taipei MRT}} & \makecell{\textbf{LargeST}} & \makecell{\textbf{Electricity}\\\textbf{Load Diagrams}} & \makecell{\textbf{CARE to Compare}\\\textbf{(Wind turbines)}} \\
\midrule
$\hat{p}_1$ & \cell{1.00}{<0.001} & \cell{1.00}{<0.001} & \cell{1.00}{<0.001} & \cell{0.92}{<0.001} \\
$\hat{p}_2$ & \cell{1.00}{<0.001} & \cell{1.00}{<0.001} & \cell{1.00}{<0.001} & \cell{1.00}{<0.001} \\
$\hat{p}_3$ & \cell{1.00}{<0.001} & \cell{1.00}{<0.001} & \cell{1.00}{<0.001} & \cell{1.00}{<0.001} \\
$\hat{p}_4$ & \cell{1.00}{<0.001} & \cell{1.00}{<0.001} & \cell{1.00}{<0.001} & \cell{1.00}{<0.001} \\
\bottomrule
\end{tabular}%
}
\caption{Estimated proportions $\hat{p}_1$--$\hat{p}_4$ with p-values (in parentheses) for hypothesis tests $H_{0,1}$--$H_{0,4}$.}
\label{tab:proportion_estimates}
\end{table}

For all four experiments, the null hypotheses $H_{0,i}$ (for $i = 1,2,3,4$) were rejected at the $0.0125$ significance level. We thus conclude that for the tested datasets, each of $\mathbf{P}_1$, $\mathbf{P}_2$, $\mathbf{P}_3$, and $\mathbf{P}_4$ occurs in \textit{most} (that is, more than 70 percent) of the tensor decompositions produced by the MINT pipeline. Since the random-selection model for each property falls inside the null region, this suggests that MINT meaningfully co-clusters time-aligned cross-series phenomena.
\begin{figure}[htbp]
    \centering
    \vspace{2pt}
    \includegraphics[width=0.8\linewidth]{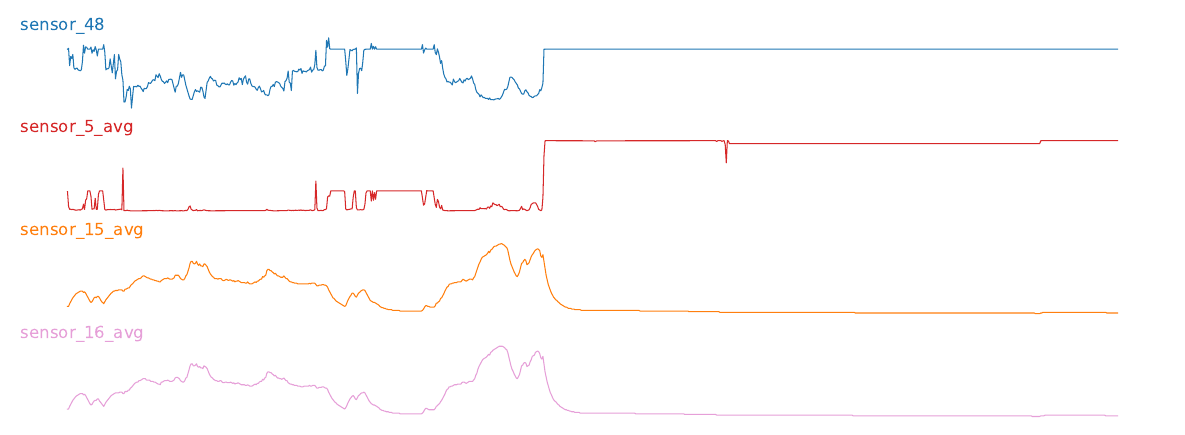}
    \vspace{2pt}
    \caption{Last week of data for four time series from the CARE to Compare dataset. Observe the time-aligned degenerate flatlines.}
    \label{fig:wind_turbines}
\end{figure}

Observe that properties $\mathbf{P}_2, \mathbf{P}_3, \text{and } \mathbf{P}_4$ were satisfied on all $50$ trials for the four datasets. However, only $92 \%$ of  trials satisfied $\mathbf{P}_1$ in the CARE to Compare wind turbine data, as compared to $100\%$ satisfaction in the other three datasets. %

We speculate that there may be two underlying reasons for the CARE dataset's unique behavior.
\begin{enumerate}
     \item Unlike the other three datasets, CARE to Compare has \textit{heterogeneous} sensors: distinct sensors (e.g. ambient temperature versus wind direction) measure different quantities. Therefore, aligned windows in $\theta_3$-transformed sensor traces are less globally similar and thus less likely to intersect the ``archetypal patterns'' (i.e. top-$s$ intervals) of a given component.
      \item Degenerate flatlines predominate within the post-anomaly tails of specific sensors, as illustrated in Figure \ref{fig:wind_turbines}. These highly time-aligned cross-series patterns, if not disturbed by noise, may become top-$s$ intervals or even dominate a high-energy CPD component, decreasing the incidence of top-$s$ intervals in the first $R$ components intersecting $\mathcal{W}$-windows.
\end{enumerate}

 While we leave formal verification of (1) and (2) and their precise linkage to incidence rates of $\mathbf{P}_1$ to future work, the results of all four hypotheses retaining their significance lend nontrivial support to the tensorized self-similarity matrix primitive as instantiated in the MINT pipeline, even in heterogeneous datasets.

\section{Conclusions and Future Work}\label{sec:conclusion}
We introduce the tensorized self-similarity matrix as a novel time series data mining primitive, instantiated through MINT as an initial method of computation. 
Where neural approaches implicitly aggregate cross-series information through learned representations, this primitive makes cross-series structure explicitly recoverable and interpretable.
Each component of the decomposition directly indexes the sensors, subsequences, and time intervals that participate in a shared pattern.
The low-rank structure recovered by tensor decomposition corresponds to interpretable cross-series phenomena, as demonstrated across MRT, wind turbine, traffic, and electricity datasets spanning multiple domains and sampling regimes.
Future work includes extension to alternative self-similarity measures (e.g. dynamic time warping, phase-space embeddings), and additional application to multivariate time series datasets to further characterize the primitive's behavior across diverse analytical settings.

\section*{Acknowledgments.}
The work at UCR was partially supported by the National Science Foundation under CAREER grant IIS \#2046086. The views and conclusions in this paper are those of the authors and should not be interpreted as representing any funding agencies.

In the final two weeks leading to submission, GPT-5.4 and Claude Fable 5 were used to debug LaTeX compilation issues, and audit a lengthy first draft of the Introduction, for grammar corrections in the Related Works.
To audit, prompts asked for sentence-by-sentence draft evaluation, to determine where content might be redundant.
No text was generated a priori.
The authors referenced the generated audit information to, \textit{by their own efforts}, rewrite sentences in a condensed manner to meet the submission page limits.

\bibliographystyle{siamplain}
\bibliography{references}
\section*{Appendices}
\appendix

\section{Code} All code used in this paper is stored in the repository at \url{https://github.com/InfinityImperium/mint-tensor-data-mining}.

\section{Bounding Chance Pass Rates}\label{app:chance_pass_rates}For a time series $T$ of $N$ length-$n$ sensor traces, let $m$ denote subsequence length and $c_k$ the Mplot pixel width. Suppose the random-selection model holds for $T$. In the single-trial procedure, a set of chosen windows $\mathcal{W}$ is constructed with coverage $f = \frac{m|\mathcal{W}|}{n}$, while fractions of $p_b, p_b,\text{and } 1-2p_b$ of the time series are transformed by $\theta_1, \theta_2,\text{ and }\theta_3$ respectively. From each of the $R \in [2 ..8]$ components of the tensor decomposition, the top-$s$ sensors and top-$s$ intervals are obtained. The experimental constants in Section \ref{sec:experiment_details} fix $c_k \approx m/3$, $f \le 0.10$,  $p_b = 0.25$, and $s = 6$. For clarity we assume $p_bN$ is an integer so that $b = p_b N$ exactly.

We now establish bounds on the four chance pass rates.\begin{proposition}
    In the random-selection model,
    \[p_1^\text{chance} \leq 0.00584.\]
\end{proposition}\begin{proof}
    Prior to tensor decomposition, $w= \frac{n \cdot f}{m}$ $\mathcal{W}$-windows are planted. An interval of length $l = m + 2c_k$ intersects a $\mathcal{W}$-window $[i ..i + m)$ if its left endpoint lies in $(i - l ..i +m)$, so it intersects one of the $w$ windows with probability at most 
    \[w \cdot \frac{(m + l)}{n} = \frac{nf}{m} \cdot\frac{2m + 2c_k}{n}= 2f \cdot (1 + \frac{c_k}{m}).\]
    For a given component, all top-$s$ intervals intersect $\mathcal{W}$-windows with probability at most $(2f \cdot (1 + \frac{c_k}{m}))^s$. Taking a union bound over $R$ components, we yield
    \[p_1^\text{chance} \le R\left(2f \cdot  (1 + \frac{c_k}{m})\right)^s.\]
    With $s  = 6$, $f \le 0.10$, $\frac{c_k}{m} \le 1/2$, and $R \le 8$ we obtain the desired inequality.
\end{proof}
\begin{proposition}
    In the random-selection model,
    \[p_2^\text{chance} \le 0.176.\]
\end{proposition}
\begin{proof}
Let $S_{\phi, i}$ denote the number of top-$s$ sensors in component $\phi$ transformed by $\theta_i$. Then under the random-selection model
\[S_{\phi,i} \sim
\begin{cases}
\text{Hypergeometric}(N, p_bN, s), & i = 1,2\\
\text{Hypergeometric}(N, (1-2p_b)N, s), & i = 3
\end{cases}.\] We also define the auxiliary variable
\[
B_{\phi, i } \sim \begin{cases}
\text{Binomial}(s, p_b), & i = 1,2\\
\text{Binomial}(s, 1-2p_b), & i = 3
\end{cases}.
\]
Using independence of components and observing that only one of $\theta_1, \theta_2, \theta_3$ can transform more than half of $\phi$'s top-$s$ sensors, we find that
\[p_2^\text{chance} = \Pr(\forall_{\phi \in [1 .. R]} \,, \exists_{i \in \{1,2,3\}}\,, S_{\phi,i} > \frac{s}{2})\]
\[= \prod_{\phi = 1}^R\sum_{i = 1}^3 \Pr(S_{\phi, i} > \frac{s}{2}).\]
By Satz 5 of Uhlmann \cite{Uhlmann1966},  
\[
\forall_{i \in \{1,2,3\}} \,, \Pr(S_{\phi, i} > c) < \Pr(B_{\phi, i} > c)\] 
when \[p_b, 1-2p_b \le \frac{c}{s-1}-\frac{c}{(s-1)(N+1)}=\frac{c}{s-1} \frac{N}{N+1}.\] These conditions are satisfied for $c = \frac{s}{2}$ in the random-selection model with the provided experiment constants ($p_b = 0.25$ and $s \le N$ an even number suffice), so we can upper-bound by the binomial tail probability.
\[
\prod_{\phi = 1}^R\sum_{i = 1}^3 \Pr(S_{\phi, i} > \frac{s}{2}) \le \prod_{\phi = 1}^R\sum_{i = 1}^3 \Pr(B_{\phi, i} > \frac{s}{2})
\]
\[
= \left(\sum_{i = 1}^3 \Pr(B_{\phi, i} > \frac{s}{2})\right)^R.
\]
With $s = 6$, $p_b = 0.25$, and $R \ge 2$, we obtain the desired inequality.
\end{proof}

\begin{proposition}In the random-selection model,
    \[p_3^\text{chance} \le 0.325.\]
\end{proposition}\begin{proof} Let $E_i$ denote the number of components where there exists a top-$s$ sensor transformed by $\theta_i$. Notice that \[
E_i \sim \text{Binomial}(R, \Pr(S_{\phi,i} \ge 1)).
\]
Hence,
\[
p_3^\text{chance} = \Pr(E_{1} \le 1)
= (1- \Pr(S_{\phi,1} \ge 1))^R \]\[+R\cdot \Pr(S_{\phi,1} \ge 1)(1- \Pr(S_{\phi,1} \ge 1))^{R-1},\]
which increases as $\Pr(S_{\phi,1} \ge 1)$ decreases. Moreover
\[
\Pr(S_{\phi,1} \ge 1) = 1-\Pr(S_{\phi,1} = 0)
\]
\[
= 1 - \frac{\binom{(1-p_b)N}{s}}{\binom{N}{s}} = 1- \prod_{i=0}^{s-1} \frac{(1-p_b)N - i}{N-i}.
\]
Every term with $i > 0$ of the product \[\prod_{i=0}^{s-1} \frac{(1-p_b)N - i}{N-i}\] increases in $N$, so $\Pr(S_{\phi,1} \ge 1)$ decreases in $N$. Thus, as $N$ increases, $p_3^\text{chance}$ increases. In fact,
\[
\lim_{N \to \infty}\Pr(S_{\phi,1} \ge 1) =  \Pr(B_{\phi,1} \ge 1) = 1-(1-p_b)^s,
\]
\[
\implies \Pr(S_{\phi,1} \ge 1) \ge 1-(1-p_b)^s.
\]
Therefore,
\[p_3^\text{chance} \le (1-p_b)^{sR} +R\cdot (1-(1-p_b)^s)(1-p_b)^{s(R-1)}.\]
With $p_b = 0.25$, $s = 6$, and $R \ge 2$, we obtain the desired inequality.
\end{proof}
\begin{table}[t]
\centering
\small
\renewcommand{\arraystretch}{1.2}
\setlength{\tabcolsep}{10pt}
\begin{tabular}{c c}
\toprule
\textbf{Rank $R$} & \textbf{Upper bound for $p_4^{\text{chance}}$} \\
\midrule
2 & 0.570 \\
3 & 0.274 \\
4 & 0.426 \\
5 & 0.226 \\
6 & 0.111 \\
7 & 0.190 \\
8 & 0.0991 \\
\bottomrule
\end{tabular}
\caption{Upper bounds for $p_4^{\text{chance}}$ as a function of the selected rank $R$.}
\label{tab:p4_chance_bounds}
\end{table}
\begin{proposition}In the random-selection model,
    \[p_4^\text{chance} \le 0.570.\]
\end{proposition}\begin{proof} Let $C_i$ denote the number of components  that $\theta_i$ corresponds to. Notice
\[C_i \sim \text{Binomial}(R, \Pr(S_{\phi,i} > \frac{s}{2})).\]
Let $p_s := \Pr(S_{\phi,3} > \frac{s}{2})$. Then
\[p_4^{\text{chance}} = \Pr(C_3 \ge \lfloor 0.7R \rfloor) = \sum_{r = \lfloor 0.7R\rfloor}^R\Pr(C_3 = r)
\]
\[
= \sum_{r = \lfloor 0.7R\rfloor}^R\binom{R}{r}\left(p_s\right)^r\left(1-p_s\right)^{R-r}.
\]
Moreover,
\[
\frac{\partial}{\partial p_s}p_4^{\text{chance}} =\binom{R}{\lfloor 0.7R\rfloor} \lfloor0.7R\rfloor p_s^{\lfloor 0.7 R\rfloor - 1} (1-p_s)^{R - \lfloor 0.7R\rfloor},
\]
which is positive, so $p_4^{\text{chance}}$ increases with $p_s$. Now, by the Uhlmann bound
\[p_s \le \Pr(B_{\phi,3} > \frac{s}{2}),\]
so with $s = 6$ and $1-2p_b = 0.5$, we can upper-bound $p_4^{\text{chance}}$ for $2 \le R \le 8$. The calculated values are listed in Table \ref{tab:p4_chance_bounds}. Taking the maximum upper bound over all ranks thus yields the desired inequality.
\end{proof}

\end{document}